\def\year{2019}\relax
\documentclass[letterpaper]{article} %DO NOT CHANGE THIS
\usepackage{aaai19}  %Required
\usepackage{times}  %Required
\usepackage{helvet}  %Required
\usepackage{courier}  %Required
\usepackage{url}  %Required
\usepackage{graphicx}  %Required

\usepackage{url}            % simple URL typesetting
\usepackage{booktabs}       % professional-quality tables
\usepackage{amsfonts}       % blackboard math symbols
\usepackage{nicefrac}       % compact symbols for 1/2, etc.
\usepackage{microtype}      % microtypography

\usepackage{amsmath}
\usepackage{mathabx}
\usepackage{amsopn}

\usepackage{algpseudocode}
\usepackage{algorithm}

\usepackage{graphicx}

\usepackage{color}
\usepackage{caption}

\usepackage{float}
\floatstyle{plaintop}
\restylefloat{table}

\usepackage{subfigure}
\usepackage{amsthm}
\usepackage{bm}

\DeclareMathOperator*{\argmax}{arg\,max}
\DeclareMathOperator*{\argmin}{arg\,min}

\newtheorem{lemma}{Lemma}

\usepackage[usestackEOL]{stackengine}[2013-09-11]

\setstackgap{L}{1.4\baselineskip}
\begin{document}
% The file aaai.sty is the style file for AAAI Press 
% proceedings, working notes, and technical reports.
%
\title{Reinforcement Learning under Threats}
\author{
Victor Gallego \\
  ICMAT-CSIC \\
  \texttt{victor.gallego@icmat.es} \\
  % examples of more authors
   \And
   Roi Naveiro \\
  ICMAT-CSIC \\
   %Address \\
   \texttt{roi.naveiro@icmat.es} \\
   \And
   David Rios Insua \\
   ICMAT-CSIC \\
   %Address \\
   \texttt{david.rios@icmat.es} \\
}
\maketitle
\begin{abstract}
In several reinforcement learning (RL) scenarios, 
mainly in security settings, 
there may be adversaries trying to interfere with the reward generating process.
In this paper, we introduce Threatened Markov Decision Processes (TMDPs), 
which provide a framework to support a decision maker against a potential adversary
in RL. Furthermore, we propose a level-$k$ thinking scheme 
resulting in a new learning framework to deal with TMDPs. After introducing our framework and
deriving theoretical
results, relevant empirical evidence is given via extensive experiments, 
showing the benefits of accounting for adversaries while the agent learns.
\end{abstract}

\section{Introduction}

\textit{Markov decision processes} (MDP) \cite{howard:dp} provide 
a mathematical framework for modeling a single agent making decisions 
while interacting within an environment. We refer to this agent as the
decision maker (DM, she). MDPs have been widely used to study 
reinforcement learning (RL) problems. More precisely, a MDP consists of a tuple $\left( \mathcal{S}, \mathcal{A}, \mathcal{T}, r\right)$ where $\mathcal{S}$ is the state space; $\mathcal{A}$ denotes the set of actions available to the agent; $\mathcal{T}: \mathcal{S} \times \mathcal{A} \rightarrow \Delta (\mathcal{S})$ is the transition distribution, where $\Delta(X)$ denotes the set of all distributions over set $X$;
and, finally, $r : \mathcal{S} \times \mathcal{A}  \rightarrow \Delta(\mathbb{R}) $
is the reward distribution (the utility the agent perceives from a given state
and action). A common approach to solving MDPs is based on Q-learning, \cite{sutton1998reinforcement}. 
In it, the agent maintains a table $Q : \mathcal{S} \times \mathcal{A} \rightarrow \mathbb{R}$ that estimates the DM's expected cumulative reward, 
iterating according to the following update equation

\begin{align}\label{eq:ql}
\begin{split}
Q(s,a) &:= (1 - \alpha) Q(s, a)  + \\ &+ \alpha \left(r(s,a) + \gamma\max_{a'} Q(s', a') \right),
\end{split}
\end{align}
where $\alpha$ is a learning rate hyperparameter
and $s'$ is the state the agent arrives at after choosing action $a$ in 
state $s$ and receiving reward $r(s,a)$. While learning, the agent
could choose actions according to a greedy policy ($\pi(s) = \argmax_a Q(s, a) $) yet it
is crucial to add stochasticity so that the agent can balance the exploration-exploitation trade-off, for instance with an $\epsilon-$greedy policy.

However, when non stationary environments are considered, as when there are other 
learning agents that interfere with the DM's rewards, Q-learning leads
to suboptimal results \cite{marl_over}. 

%Other learning agents may interfere with the reward distribution, cooperating or competing with the supported DM. 

\subsection{Related Work}

Several extensions of Q-learning in multi-agent settings have been developed 
in the literature, including minimax-Q \cite{littman1994markov}, 
Nash-Q \cite{hu2003nash} or friend-or-foe-Q \cite{littman2001friend},
to name but a few. 
We propose here to extend Q-learning from an Adversarial Risk Analysis (ARA), 
\cite{rios2009adversarial} perspective, in particular, through 
a level-$k$ scheme \cite{stahl1994experimental}, \cite{stahl1995players}.

Within the bandit literature, the celebrated \cite{auer1995gambling} introduced 
a non-stationary setting in which the reward process is controlled by an adversary.
The adversarial machine learning literature has predominantly focused on the supervised setting \cite{biggio2017wild}. 
Other recent works tackle the problem of adversarial examples in RL \cite{huang2017adversarial,lin2017tactics} though they focus on visual inputs. 

Moreover, previous game-theoretical approaches to this problem have focused on modeling the whole multi-agent system as a game. Instead we shall face the problem of prescribing decisions to a single agent
versus her opponents, augmenting the MDP to account for potential adversaries. We present such variant of MDPs,
which we call Threatened MDPs (TMDPs), in the next section.

\section{Threatened MDPs}

In similar spirit to other reformulations of MDPs such as Constrained Markov Decision Processes (CMDP) \cite{altman1999constrained} or Configurable Markov Decision Processes (Conf-MDP) \cite{2018arXiv180605415M}, we propose an augmentation of the MDP
to account for the presence of adversaries. In this paper, we restrict to the case of a DM facing a single opponent (he), leaving the extension to 
a setting with multiple adversaries for future work.

A \emph{Threatened Markov Decision Process} (TMDP) is a tuple $\left( \mathcal{S}, \mathcal{A}, \mathcal{B}, \mathcal{T}, r, p_A \right) $ 
in which $\mathcal{S}$ is the state space; $\mathcal{A}$ denotes the set of actions available to the supported agent; $\mathcal{B}$ designates the set of threat actions, or actions
available to the adversary; $\mathcal{T}: \mathcal{S} \times \mathcal{A} \times \mathcal{B} \rightarrow \Delta(\mathcal{S})$ is the transition distribution; 
$r : \mathcal{S} \times \mathcal{A} \times \mathcal{B} \rightarrow \Delta(\mathbb{R}) $ is the reward distribution (the utility the agent perceives from a given state
and action pair);  and $p_A (b | s)$ models the beliefs that the DM has about
his opponent move, i.e., a distribution over $\mathcal{B}$ for each state $s \in \mathcal{S}$.

We propose to replace the standard Q-learning rule (Eq.\ref{eq:ql})  by 

\begin{align}\label{eq:lr}
\begin{split}
Q(s, a, b) &:= (1 - \alpha)Q(s, a, b) + \\ &+ \alpha \left( r(s,a,b) + \gamma \max_{a'} \mathbb{E}_{p_A(b|s')} \left[ Q(s',a',b)  \right]  \right)
\end{split}
\end{align}
and compute its expectation over the opponent's action argument

\begin{align}
Q(s,a) := \mathbb{E}_{p_A(b|s)} \left[ Q(s,a,b) \right].
\end{align}

This may be used to compute an $\epsilon-$greedy policy for the DM, i.e., choosing
with probability $(1-\epsilon)$ the action $a = \argmax_a  \left[ Q(s,a)  \right] $ or a uniformly random action with probability $\epsilon$ when 
the DM is at state $s$. 
Appendix \ref{app:proofs} contains two lemmas showing
%In what follows we introduce two lemmas showing
that the previous update rules are fixed point iterations of contraction mappings.

% \begin{proposition}\label{prop:1}
% Assuming opponent acts according to $p_A$, then the learning rule \ref{eq:lr} converges to the optimal $Q$-function $Q^*$, and
% $$
% \epsilon-\argmax \mathbb{E}_{p_A(b|s)} \left[ Q^*(s,a,b)  \right] 
% $$ is an optimal policy.
% \end{proposition}
% \begin{proof}
% See Appendix.
% \end{proof}

However, in real life scenarios there will be uncertainty regarding the adversary's policy $p_A (b | s)$. Therefore, we propose using a level-$k$ scheme \cite{rios2009adversarial}
 to learn the opponent model. 
In general, we consider both the DM and the adversary 
as rational agents that aim to maximize their respective expected cumulative rewards, though we start with a case in which the adversary is considered non-strategic (Section \ref{sec:non}). Then, we go up a level in the level-$k$ hierarchy, considering the adversary a level-1 agent and the DM a level-2 one (Section \ref{sec:k}).

%Fictitious play, etc.

\subsection{Non-strategic opponent}\label{sec:non}
%\subsection{NON-STRATEGIC OPPONENT}\label{sec:non}

We begin by considering a stateless setting. The Q-function may be written then
as $Q(a_i,b_j)$, with $a_i \in \mathcal{A}$ the action chosen by the DM, and $b_j \in \mathcal{B}$ the action chosen by the adversary. We assume that the supported DM is a joint action learner (i.e., she observes her opponent's actions after he has committed them). At every iteration, the DM shall choose her action maximizing her expected cumulative reward. However, she needs to predict the action $b_j$ chosen by her opponent. A typical option is to model her adversary using fictitious play (FP), i.e., she may compute the expected utility of action $a_i$ via

\[ \psi(a_i) = \sum_{b_j \in \mathcal{B}} Q(a_i, b_j) p_{A}(b_j) \]
where $p_A (b_j)$ reflects $A$'s beliefs about her opponent's actions and is computed using the empirical frequencies of the opponent past plays.
Then she may choose the action $a_i \in \mathcal{A}$ that
maximizes her expected utility. In the following sections, we refer to this variant as FPQ-learning.

As described in \cite{rios1}, it is possible to re-frame fictitious play from a Bayesian perspective. Let $p_j$ be the probability that the opponent chooses action $b_j$. We may place a Dirichlet prior $(p_1 , \ldots, p_n) \sim \mathcal{D}(\alpha_1,\ldots,\alpha_n)$. Then, the posterior has the analytical form $\mathcal{D}(\alpha_1 + h_1,\ldots,\alpha_n + h_n)$, with  $h_i$ being the count of action $b_i$, $i=1,...,n$. If we denote the posterior density function as $f(p|h)$, then the DM would choose the action $a_i$ maximizing her expected utility, which now takes the form

\begin{eqnarray*}
& \psi(a_i) & = \int \left[\sum_{b_j \in \mathcal{B}} Q(a_i, b_j) p_j\right] f(p|h) dp \\
&=& \sum_{b_j \in \mathcal{B}} Q(a_i, b_j) \mathbb{E}_{p|h}[p_i] \propto  \sum_{b_j \in \mathcal{B}} Q(a_i, b_j) (\alpha_i + h_i).
\end{eqnarray*}
The Bayesian perspective may benefit the convergence of Q-learning, as we may include prior information about the adversary behavior when relevant.

Generalizing the previous approach to account for states is straightforward. Now the Q-function has the form $Q(a_i, b_j, s)$, where $s$ is the state of the TMDP. The DM may need to asses probabilities of the form $p_A(b_j | s)$, since it is natural to expect that her opponent behaves differently depending on the state of the game and, consequently, depending also on previous actions. As before, the supported DM may choose her action at state $s$ by maximizing

\[ \psi_s(a_i) = \sum_{b_j \in \mathcal{B}} Q(a_i, b_j,s) p_{A}(b_j|s).  \]

Since the state space $\mathcal{S}$ may be huge (or even continuous), keeping track of $p_A(b_j|s)$ may incur in prohibitive memory costs. Bayes rule may turn out to be useful, using

\[
p_{A}(b_j| s) \,\, \propto \,\, p(s| b_j)p(b_j).
\]
\cite{tang2017exploration} propose an efficient method using a hash table or a bloom filter to maintain a count of the number of times an agent visits each state $s$, $p(s)$. This is only used in the context of single-agent RL to assist for better exploration of the environment. We propose to keep track of $|\mathcal{B}| = n$ bloom filters, one for each distribution $p(s|b_j)$, for tractable computation of the opponent's intentions in the TMDP setting.

The previous scheme may be transparently integrated with the Bayesian paradigm: we only need to store an additional array with the Dirichlet prior parameters $\alpha_i$, $i=1,\ldots, n$ 
for the $p(b_j)$ part. Potentially, we could store initial pseudocounts as priors for each $b_j | s$ initializing the bloom filters with the corresponding parameter values.

If we assume the opponent to have memory of the previous stage actions, we could straightforwardly extend the previous scheme using the concept of mixtures of Markov chains, as described in \cite{raftery1985model}. For example, in case the opponent belief model
is $p_{A}(b_t | a_{t-1}, b_{t-1}, s_t)$, so that the adversary recalls 
 the previous actions $a_{t-1}$ and $b_{t-1}$, it could be factorized as a mixture 

\begin{align*}
\begin{split}
p_{A}(b_t | a_{t-1}, b_{t-1}, s_t) &= w_1 p_{A}(b_t | a_{t-1})  +\\ &+ w_2 p_{A}(b_t | b_{t-1})
 + w_3 p_{A}(b_t | s_t).
 \end{split}
\end{align*}
Then, if we allow for longer memories, instead of an exponential growth in the number of parameters, the complexity can be linearly controlled.

To conclude this section, we shall note that the described scheme is \emph{model agnostic}, i.e., it does not matter if we represent the Q-function using a look-up table or a deep neural network (DQN), so we expect it to be usable in both shallow and deep multi-agent RL settings.

\subsection{Level-$k$ thinking}\label{sec:k}
%\subsection{LEVEL-$k$ THINKING}\label{sec:k}

The previous section described how to model a level-0 opponent, i.e. a non strategic opponent, which can be practical in several scenarios. However, if the opponent is strategic, he may model the supported DM as a level-0 thinker, thus making the adversary a level-1 thinker. This chain can go up to infinity, so we will have to deal with modeling the opponent as a level-$k$ thinker, with $k$ bounded by the computational or cognitive resources of the DM.

To deal with it, we introduce a hierarchy of TMDPs in which 
 $\emph{TMDP}_{i}^k$ refers to the TMDP that agent $i$ needs to optimize,
 while considering its rival as a level-$(k-1)$ thinker.
Thus, we have the following process:

\begin{itemize}
\item If the supported DM is a level-1 thinker, she may optimize for $ \emph{TMDP}_{A}^1 $. She then models B as a level-0 thinker (using Section \ref{sec:non}).
\item If the supported DM is a level-2 thinker, she may optimize for $ \emph{TMDP}_{A}^2 $. She models B as a level-1 thinker. Consequently, this ``modeled" B optimizes $ \emph{TMDP}_{B}^1 $, and while doing so, he models the DM as level-0 (Section \ref{sec:non}).
\item In general, we have the chain of TMDPs: $$ \emph{TMDP}_{A}^k \rightarrow \emph{TMDP}_{B}^{k-1} \rightarrow \cdots \rightarrow \emph{TMDP}_{B}^{1}.$$
\end{itemize}
Exploiting the fact that we are in a repeated interaction setting (and by assumption that both agents can observe all past committed decisions and obtained rewards), each agent may  estimate their counterpart's Q-function, $\hat{Q}^{k-1}$: 
if the DM is optimizing $\emph{TMDP}_A^k$, she will keep her own Q-function (we refer to it as $Q_k$), and also an estimate $\hat{Q}_{k-1}$, of her opponent's Q-function. This estimate may be computed by optimizing $\emph{TMDP}_B^{k-1}$ and so on until $k=1$.
Finally, the top level DM's policy is given by

\[
\argmax_{a_{i_k}} Q_k(a_{i_k}, b_{j_{k-1}}, s),
\]
where $b_{j_{k-1}}$ is now given by 

\[
\argmax_{b_{j_{k-1}}} \hat{Q}_{k-1}(a_{i_{k-2}}, b_{j_{k-1}}, s)
\]
and so on, until we arrive at the induction basis (level-1) in which the opponent may be modeled using the fictitious play approach from Section \ref{sec:non}.

\begin{algorithm*}[!ht]
\begin{algorithmic}
\Require $Q_A$, $Q_B$, $\alpha_A, \alpha_B$ (DM and opponent Q-functions and learning rates, respectively).
\State Observe  transition $(s, a, b, r_A, r_B, s')$ from the TMDP environment
\State $Q_B(s,b,a) := (1 - \alpha_B)Q_B(s,b,a)  + \alpha_B (r_B + \gamma \max_{b'} \mathbb{E}_{p_B(a'|s')} \left[ Q_B(s',b', a') \right] )$ \Comment Level-1
\State Compute B's estimated $\epsilon-$greedy policy $p_A(b|s')$ from $Q_B(s,b,a)$
%\State $b^* \leftarrow \argmax_{b} \mathbb{E}_{p_B(a|s')} \left[ Q_B(s',b,a) \right]$ \Comment Compute B's most probable action $b^*$ at $s'$
\State $Q_A(s,a,b) := (1 - \alpha_A)Q_A(s,a,b) + \alpha_A (r_A + \gamma \max_{a'} \mathbb{E}_{p_A(b'|s')} \left[ Q_A(s',a',b')) \right] $ \Comment Level-2 
\end{algorithmic}
\caption{Level-2 thinking update rule}
\label{alg:l2ur}
\end{algorithm*}

Note that in the previous hierarchy of policies the decisions are obtained in a greedy, deterministic manner (i.e. just by maximizing the lower level $\hat{Q}$ estimate). We may gain insight from the Bayesian / Risk Analysis communities by adding uncertainty to the policy at each level. For instance, at a certain level in the hierarchy, we could consider $\epsilon-$greedy policies that with probability $1-\epsilon$ choose an action according to the previous scheme, and with probability $\epsilon$ select a random action. Thus, we may impose distributions $p_k(\epsilon)$ at each level $k$ of the hierarchy. The mean of $p_k(\epsilon)$ may be an increasing function with respect to the level $k$ to account for the fact that in upper levels of thinking the uncertainty is higher. Other approaches to add uncertainty to the policies are left for future work.

\begin{figure}
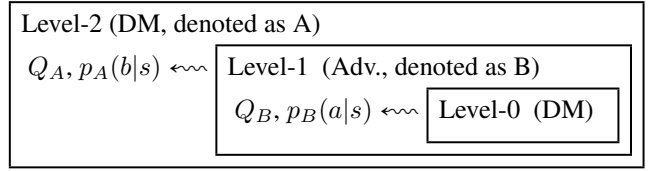

\centering
\stackinset{c}{.5in}{t}{.73in}{%
  \fboxrule=0pt\relax\framebox[2in][t]{%
  }}{\fboxrule=.75pt%
  \fbox{\stackunder{Level-2%
   \hspace*{\fill} (DM, denoted as A) }%
    {
    $Q_A$,
    $p_A(b | s) \leftsquigarrow$
    \fbox{\stackunder{Level-1 \hspace*{\fill} (Adv., denoted as B) }%
      {
      $Q_B$,
    $p_B(a | s) \leftsquigarrow$
      \fbox{\stackunder{Level-0 \hspace*{\fill} (DM) }%
        {}
        }}}}}
}
\caption{Level-$k$ thinking scheme, with $k=2$}\label{fig:lev2_scheme}
\end{figure}

% \begin{table*}[ht]
%   \caption{Level-$k$ Q-learning  { \color{blue} CHANGEME }}
%   \label{sample-table}
%   \centering
%   \begin{tabular}{ll}
%     \toprule
%     %\multicolumn{2}{c}{Part}                   \\
%     \cmidrule(r){1-2}
%     Level     & Learning rule     \\
%     \midrule
%     2 & $Q_2(s,a,b) \leftarrow (1 - \alpha)Q_2(s,a,b) + \alpha \left( r_A + \gamma \max_a \mathbb{E}_{p_A(b|s')} \left[ Q_2(s',a,b)  \right]  \right) $    \\
%     1     & $Q_1(s,a,b) \leftarrow (1 - \alpha)Q_1(s,a,b) + \alpha \left( r_B + \gamma \max_b \mathbb{E}_{p_B(a|s')} \left[ Q_1(s',a,b)  \right]  \right) $       \\
%     0     & random agent        \\
%     \bottomrule
%   \end{tabular}
% \end{table*}

Algorithm \ref{alg:l2ur} specifies the approach 
for a level-2 DM. Because she is a level-2 DM, we need to account for her Q-function, $Q_A$ (equivalently $Q_2$ from before), and that of her opponent (who will be level-1), $Q_B$ (equivalently $\hat{Q}_1)$. Figure \ref{fig:lev2_scheme} provides a schematic view of the dependencies.

\section{Experiments and Results}

To illustrate the TMDP's and level-$k$ reasoning framework, we consider two sets of experiments: repeated matrix games, with and without memory, and the adversarial environment proposed in \cite{leike2017ai}. All the code is released at \url{https://github.com/vicgalle/ARAMARL}. 
The interested reader might check the previous repository or the Supplementary Material \ref{sec:exp_det} for experimental setup details.

\subsection{Repeated matrix games}
%\subsection{REPEATED MATRIX GAMES}

\subsubsection{Memoryless Repeated Matrix Games}

As an initial baseline, we focus on the stateless version of a TMDP. We consider the classical Iterated Prisoner's Dilemma (IPD) \cite{axelrod84}, and 
analyze the policies learned by the supported DM, who will be the row player, against several kinds of opponents. Table \ref{tab:payoffIPD} shows the reward bimatrix $(\mathbf{r^A}, \mathbf{r^B})$ for the Prisoner's Dilemma. To construct the
iterated game, we set the discount factor $\gamma = 0.96$ in the experiments so agent $ i \in \lbrace A, B \rbrace $ aims at optimizing $\sum_{t=0}^{\infty} \gamma^t r^i_{t}$.

\begin{table}[h]
\begin{center}
\begin{tabular}{c|c|c}
\hline
 & C & D \\
\hline
C & (-1, -1) & (-3, 0) \\
\hline
D & (0, -3) & (-2, -2)  \\
\hline
\end{tabular}
\end{center}
\caption{Payoff Matrix of Prisoners' Dilemma}
\label{tab:payoffIPD}
\vspace{-2ex}
\end{table}

%First, we consider a non-strategic and stationary opponent. The parameter $p \in \left[ 0, 1 \right]$ will denote the probability that the opponent will choose to cooperate at each step. The supported DM does not have knowledge about this parameter. We compare our proposed method with a standard independent Q-learning solution modeling the supported DM.
%
%We checked that both the FPQ-learner and the independent Q-learner reach the same policy. From the point of view of the DM, her environment is stationary (due to the stationarity in her adversary), so the single-agent RL solution converges to the optimal solution. Since the opponent is random but stationary, after sufficient observations of opponent' actions DM's estimate of the parameter $p$ will be centered around its true value.

To start with, we consider that the opponent is an independent-Q learner (i.e., he uses the standard Q-function from single-agent RL and Eq. \ref{eq:ql} as learning rule). Figures \ref{fig:QvsQ} and \ref{fig:FPQvsQ} depict the utilities obtained over time, 
in cases where we model the DM as another independent Q-learner (Fig. \ref{fig:QvsQ}) or as a joint Q-learner with fictitious play (FPQ-learner), Fig. \ref{fig:FPQvsQ}.
Note that the FP playing solution (level-1) converges to the Nash equilibrium. The DM reaches the equilibrium strategy first, becoming stationary to her opponent, and thus
pulling him to play towards the equilibrium strategy. In contrast, the opponent-unaware solution would remain exploitable by another adversary (i.e., independent Q-learning does not converge). Also note that in Fig. \ref{fig:QvsQ} the variance is much bigger due to the inability of the basic Q-learning solution to deal with a non-stationary environment.

\begin{figure*}%
\centering
\subfigure[Q-learner vs Q-learner]{%
  \label{fig:QvsQ}%
  \includegraphics[height=1.8in]{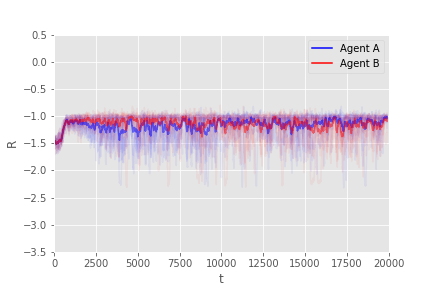}}%
  \subfigure[FPQ-learner (blue) vs Q-learner (red)]{%
  \label{fig:FPQvsQ}%
  \includegraphics[height=1.8in]{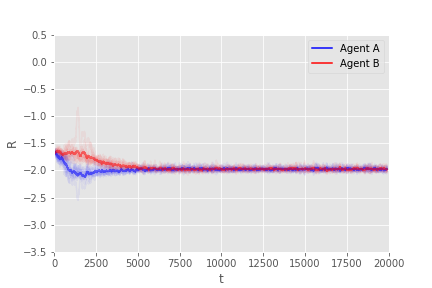}}%
  \caption{Rewards obtained in IPD. We plot the trajectories of 10 simulations with shaded colors. Darker curves depict mean rewards along the 10 simulations. }
  
 \subfigure[Q-learner vs Q-learner]{%
\label{fig:QvsQ_SH}%
\includegraphics[height=1.8in]{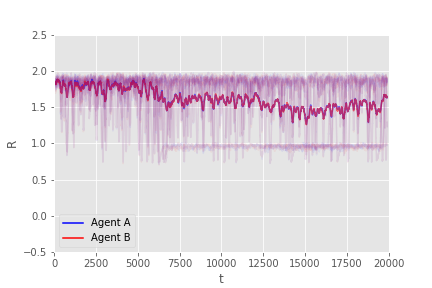}}%
\subfigure[FPQ-learner (blue) vs Q-learner (red)]{%
\label{fig:FPQvsQ_SH}%
\includegraphics[height=1.8in]{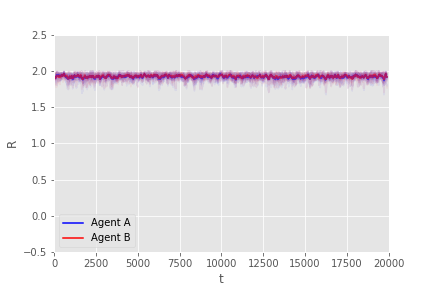}}%
\caption{Rewards in the iterated stag hunt game}\label{fig:ISH}

\subfigure[Q-learner vs Q-learner]{%
\label{fig:QvsQ_C}%
\includegraphics[height=1.8in]{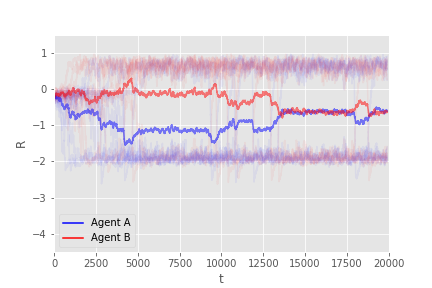}}%
\subfigure[FPQ-learner (blue) vs Q-learner (red)]{%
\label{fig:FPQvsQ_C}%
\includegraphics[height=1.8in]{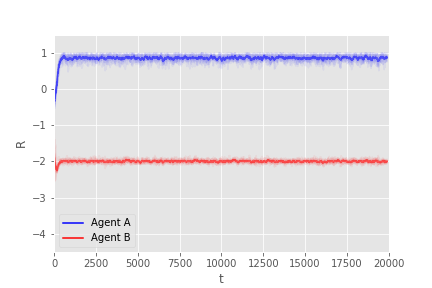}}%
\caption{Rewards in the iterated chicken game}\label{fig:IC}

\end{figure*}

We turn to another social dilemma game in which both agents must coordinate to maximize their rewards, the Stag Hunt game,
with payoff matrix shown in Table \ref{tab:payoffSG}. We focus on 
its iterated version referred to as ISH.

\begin{table}[h!]
\begin{center}
\begin{tabular}{c|c|c}
\hline
 & C & D \\
\hline
C & (2, 2) & (0, 1) \\
\hline
D & (1, 0) & (1, 1)  \\
\hline
\end{tabular}
\end{center}
\caption{Payoff Matrix of Stag Hunt}
\label{tab:payoffSG}
\vspace{-2ex}
\end{table}
We repeated the same experimental setting as in the IPD and report the results in Figure \ref{fig:ISH}. Once again, the independent learning solution cannot
tackle the non-stationarity of the environment, so it oscillates between the two Nash equilibria (C,C) and (D,D) without a clear convergence to one of them (Fig. \ref{fig:QvsQ_SH}). On the other hand, the FPQ-learner converges earlier to the socially optimal policy. Then, the environment becomes essentially stationary for its opponent, who also converges to that policy.

% \begin{figure*}[h!]
% \centering
% \subfigure[Q-learner vs Q-learner]{%
% \label{fig:QvsQ_SH}%
% \includegraphics[height=1.8in]{figures/QvsQ_SH}}%
% \subfigure[FPQ-learner (blue) vs Q-learner (red)]{%
% \label{fig:FPQvsQ_SH}%
% \includegraphics[height=1.8in]{figures/FPQvsQ_SH}}%
% \caption{Rewards obtained in the iterated stag hunt game}\label{fig:ISH}
% \end{figure*}

The last social dilemma that we consider is the Chicken game, with payoff matrix in Table \ref{tab:payoffC}. This game has two pure Nash equilibria (C, D) and (D,C).

\begin{table}[h!]
\begin{center}
\begin{tabular}{c|c|c}
\hline
 & C & D \\
\hline
C & (0, 0) & (-2, 1) \\
\hline
D & (1, -2) & (-4, -4)  \\
\hline
\end{tabular}
\end{center}
\caption{Payoff Matrix of Chicken}
\label{tab:payoffC}
\vspace{-2ex}
\end{table}
% \begin{figure*}[h!]
% \centering
% \subfigure[Q-learner vs Q-learner]{%
% \label{fig:QvsQ_C}%
% \includegraphics[height=1.8in]{figures/QvsQ_C}}%
% \subfigure[FPQ-learner (blue) vs Q-learner (red)]{%
% \label{fig:FPQvsQ_C}%
% \includegraphics[height=1.8in]{figures/FPQvsQ_C}}%
% \caption{Rewards obtained in the iterated chicken game}\label{fig:IC}
% \end{figure*}
Results are reported in Figure \ref{fig:IC}. Figure \ref{fig:QvsQ_C} depicts again the ill convergence due to lack of opponent awareness in the independent Q-learning method. We noted that the instabilities continued cycling even after the limit in the displayed graphics. On the other hand, the DM with opponent modeling has an advantage and converges to her optimal Nash equilibrium (D,C) (Fig. \ref{fig:FPQvsQ_C}).

{\color{black}
In addition, we also consider another kind of opponent to show our framework can adapt to it. We consider an adversary that learns according to the WoLF-PHC algorithm \cite{bowling2001rational}. Figure \ref{fig:L1vsWoLF_C} depicts a FP-Q learner (level-$1$) against this adversary, where the latter exploits the former. However, if we go up in the level-$k$ hierarchy and model the DM as a level-$2$ Q-learner, she outperforms her opponent (Fig. \ref{fig:L2vsWoLF_C}).}

\begin{figure*}%
\centering
\subfigure[FPQ-learner (blue) vs WoLF-learner (red)]{%
  \label{fig:L1vsWoLF_C}%
  \includegraphics[height=1.8in]{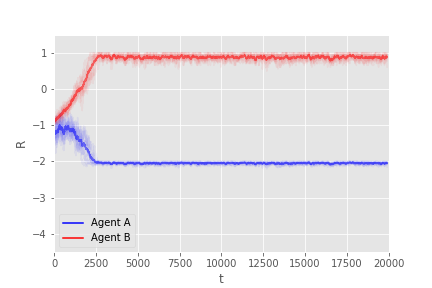}}%
  \subfigure[L2Q-learner (blue) vs WoLF-learner (red)]{%
  \label{fig:L2vsWoLF_C}%
  \includegraphics[height=1.8in]{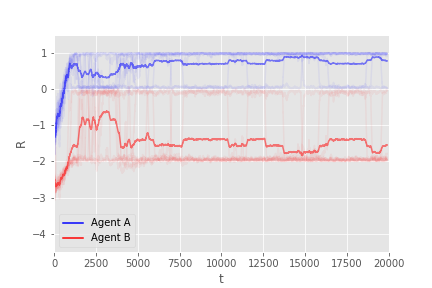}}%
  \caption{Rewards obtained in the iterated chicken game against a WoLF-PHC adversary. }
\end{figure*}

\subsubsection{Repeated Matrix Games With Memory}

In this section we give both players memory of past actions, in order to account for TMDP's with different states. We can augment the agents to have memory of the past $K$ joint actions taken. However, \cite{press2012iterated} proved that agents with a good 
memory-1 strategy can effectively force the iterated game to be played as memory-1, ignoring larger play histories. Thus, we resort to memory-1 iterated games here.

We may model the memory-1 IPD as a TMDP, in which the state $\mathcal{S}$ consists of elements of the form

$$
s_t = (a_{t-1}, b_{t-1}), \qquad t > 0
$$
describing the previous joint action, plus the initial state $s_0$ in which there is no prior action. Note that now, the DM's policy is conditioned on $\mathcal{S}$, so it may be fully specified by the $|\mathcal{S}|$ probabilities $\pi(C | CC), \pi(C | CD), \pi(C | DC), \pi(C | DD), \pi(C | s_0)$.

We assume an stationary adversary playing TitForTat (TFT), i.e. replicating the opponent's previous action, \cite{axelrod84}. He will compete
with either another agent playing FP, or with a memory-1 agent also playing FP. In Figure \ref{fig:Mem1} we represent the utilities perceived by these agents in both duels. As can be seen, a memoryless FPQ player cannot learn an optimal policy, and forces the TFT agent to play defect. In contrast, augmenting this agent to have memory of the previous move allows him to learn the optimal policy (TFT), that is, he learns to cooperate.

\begin{figure}%
\centering
\includegraphics[scale=0.5]{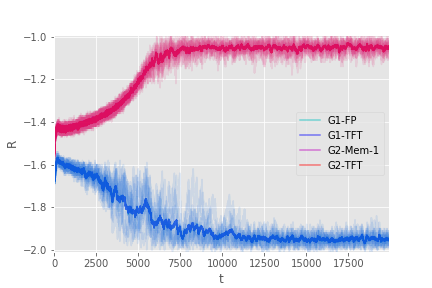}%
\caption{Rewards obtained for two different iterated games: TFT player vs FPQ memoryless player (G1) and TFT player vs FPQ memory-1 player (G2) }\label{fig:Mem1}
\end{figure}

% We believe that a high level $k$ (i.e., fictitious play versus opponent-unawareness) should be needed in order to internalize an adversary playing TFT.

% \subsubsection{Potential experiments}

% Maybe test also on some coordination games, such as the repeated variant of Stag Hunt \cite{skyrms2004stag}?

% Then, we may compute some social metrics (such as social utility $\frac{R_1 + R_2}{2}$) as introduced by \cite{perolat2017multi} and see if scores are better with a high k-level instead of with just independent Q-learning. 

% (i.e., to be a good social individual you have to model opponents)

\subsection{AI Safety Gridworlds}
%\subsection{AI SAFETY GRIDWORLDS}

A suite of RL safety benchmarks was recently introduced in \cite{leike2017ai}. We focus on the \emph{friend or foe} environment, in which the supported DM needs to travel a room and choose between two identical boxes, hiding positive and negative rewards, respectively. This reward assignment is controlled by an adaptive adversary. Figure \ref{fig:friendorfoe} shows the initial state in this game. The blue cell depicts
the DM's initial state, gray cells represent the walls of the room. Cells 1 and 2 depict the adversary's targets, who will decide which one will hide the positive reward.
This game may also be interpreted as a spatial Stackelberg game, in which the
adversary is planning to attack one of two targets, and the defender (DM) will obtain a positive reward if she travels to the chosen target. Otherwise, she will miss the attacker and will incur in a loss.

\begin{figure}
   \centering
   \includegraphics[scale=0.5]{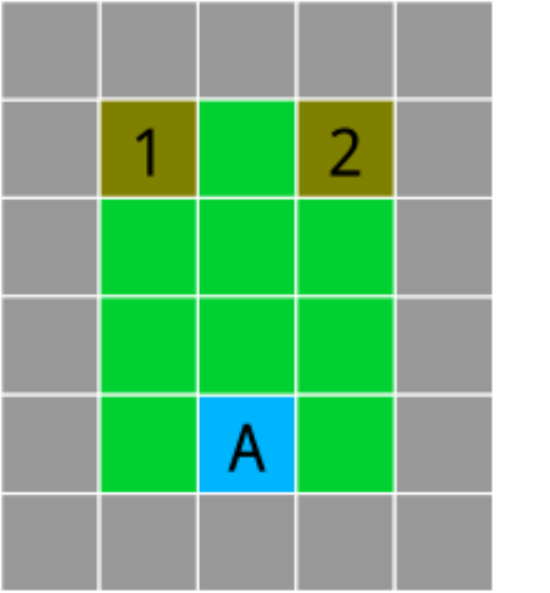}
    %\fbox{\rule[-.5cm]{0cm}{4cm} \rule[-.5cm]{4cm}{0cm}}
   \caption{The \emph{friend or foe} environment from the AI Safety Gridworlds benchmark. Figure taken from \cite{leike2017ai}. } \label{fig:friendorfoe}
 \end{figure}

As shown in \cite{leike2017ai}, the \emph{deep Q-network} (and similarly the independent tabular Q-learner as we will show) fails to achieve optimal results because the reward process is controlled by the adversary. We show that by explicitly modeling the adversary we actually
improve Q-learning methods to achieve optimal utilities.

%%%%%%%%%%%%%%%%%%%%%%%%%%%%%%%%%%%%%%%%%
\subsubsection{Stateless Variant}\label{sec:statv}

We first consider a simplified environment with a singleton state and two actions. In a similar spirit to \cite{leike2017ai}, the adaptive opponent estimates the DM's actions using an exponential smoother. Let $\bm{p} = (p_1, p_2)$ be the probabilities
with which the DM will choose targets 1 or 2, respectively,
as estimated by the opponent. Then, at every iteration he updates his knowledge
through 

$$
\bm{p} := \alpha \bm{p} + (1 - \alpha) \bm{a}
$$
where $0 < \alpha < 1$ is a learning rate, unknown from the DM's point of view, and $\bm{a} \in \lbrace (1, 0), (0, 1) \rbrace$ is a one-hot encoded vector indicating whether the DM has chosen targets 1 or 2. We consider an adversarial opponent which places the positive reward in target $t = \argmin_i (\bm{p})_i$.
%Then, we consider three different environments, depending on the type of adversary:
% \begin{itemize}
% \item \emph{Friendly opponent.} Places the positive reward in target $t = \argmax_i (\bm{p})_i$.
% \item \emph{Adversarial opponent.} Places the positive reward in target $t = \argmin_i (\bm{p})_i$.
% \item \emph{Neutral opponent.} Places the positive reward in target 1 with constant probability $p$.
% \end{itemize}
%As an example, at the beginning of a game, the opponent has estimates $\bm{p} = (0.5, 0.5)$ of the preferred target for the DM. If she chooses target 1, then opponent's estimate will increase $p_1$. Henceforth, in the next round he will place the positive reward in target 1, if he is of the friendly type, or in target 2 if he is an adversary.
As an example, in the beginning of a game, the opponent has estimate $\bm{p} = (0.5, 0.5)$ of the preferred target for the DM. If she chooses target 1, then opponent's estimate of $p_1$ will increase. Henceforth, in the next round he will place the positive reward in target 2.

%{\color{blue} ver qué se hace finalmente con lo de dirichlet con olvido}

Since the DM has to deal with an adaptive adversary, we introduce a modification to the FP-Q learning algorithm. Leveraging the property that the Dirichlet distribution is a conjugate prior of the Categorical distribution, a modified update scheme is proposed in Algorithm \ref{alg:duwff}. 

\begin{algorithm}
\begin{algorithmic}
\State Initialize pseudocounts $ \bm{\alpha^0} = (\alpha^0_1, \ldots, \alpha^0_K)$
\For{$t = 1, \ldots, T $}
\State $\bm{\alpha^t} = \lambda \bm{\alpha^{t-1}}$ \Comment Reweight with factor $0 < \lambda < 1$
\State Observe opponent action $b^t_i, i \in \lbrace b_1, \ldots, b_K \rbrace$
\State $\alpha^t_i = \alpha^{t-1}_i + 1$ \Comment Update posterior
\State $\alpha^t_{-i} = \alpha^{t-1}_{-i}$
\EndFor
\end{algorithmic}
\caption{Dirichlet updating with a forget factor}
\label{alg:duwff}
\end{algorithm}

 It essentially allows to account for the last $\frac{1}{1 - \lambda}$ opponent actions, instead of weighting all observations equally.
For the case of a level-2 defender, as we do not know 
the actual rewards of the adversary (who will be modeled as a level-1 learner), 
 we  may model it as in a zero-sum scenario, i.e. $r_B = -r_A$. Other reward scalings for $r_B$ were also considered, though they did not qualitatively affect
 the results (See Supplementary Material \ref{sec:adv_rs}).

\begin{figure}%
\includegraphics[scale=0.55]{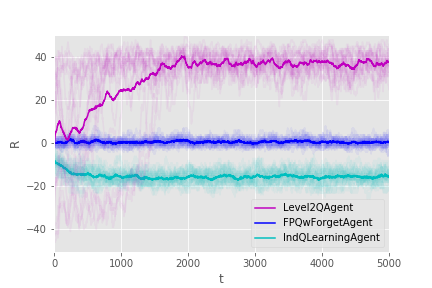}%
\caption{Rewards against the adversarial opponent}\label{fig:4C_adv}
\end{figure}

Results are displayed in Figure \ref{fig:4C_adv}. We considered four types of defenders: opponent-agnostic Q-learner, a level-1 DM with forget and a level-2 agent. The first one is exploited by the adversary and, therefore, achieves suboptimal results. In contrast, the level-1 DM with forget effectively learns an stationary optimal policy (reward 0). Finally, the level-2 agent learns to exploit the adaptive agent achieving positive reward. 

%Even though the adversary is not exactly a level-1 Q-learner, making the DM a level-2 agent gives sufficient advantage.
Note that the actual adversary behaves differently from how the DM models him, i.e. he is not exactly a level-1 Q-learner. Even so, modeling him as a level-1 agent gives the DM sufficient advantage.

%Maybe from the safety part cite something from \cite{maliciousAIreport} and/or \cite{amodei2016concrete}?

% \subsection{Markov Security Games}

% An extension of Stackelberg or security games to spatial domains was introduced in \cite{klimamarkov}. They extend the single-agent Q-learning algorithm with an adversarial policy selection inspired by the EXP3 rule from the \emph{adversarial multi-armed bandits} (AMAB) framework \cite{auer1995gambling}. Though robust, their approach does not explicitly model an adversary, nor is amenable to function approximation of the Q-function, thus scaling poorly as the number of states increases.

% We test ...

%\subsection{Pong's Dilemma?}

%Or some other 2-player game..

%\subsection{3D control tasks}

%We could apply to some robot locomotion tasks under adversaries, such as the settings introduced in \cite{bansal2017emergent}.

%%%%%%%%%%%%%%%%%%%%%%%%%%%%%%%%%%%%%%%%%%
\subsubsection{Spatial Variant}

We now compare the independent Q-learner and a level-$2$ Q-learner against the same adaptive opponent in a spatial gridworld domain, Figure \ref{fig:friendorfoe}. Targets' rewards 
are delayed until the DM arrives at one of the respective locations, obtaining $\pm 50$ depending on the target chosen by the adversary. Each step is penalized with a reward of -1 for the DM. Results are displayed in Figure \ref{fig:4C_gridworld}. Once again, the independent Q-learner is exploited by the adversary, getting even more negative rewards than in Figure \ref{fig:4C_adv} due to the penalty taken at each step. In contrast, the level-2 agent is able to approximately estimate the adversarial behavior, modeling him as a level-1 agent, thus being able to obtain positive rewards.  

\begin{figure}%
\includegraphics[scale=0.55]{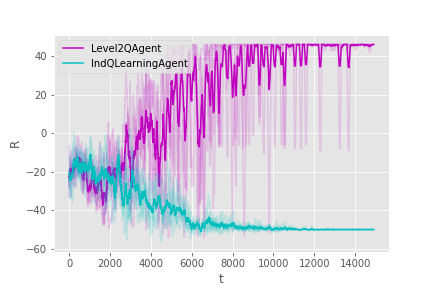}%
\caption{Rewards against the adversarial opponent in the spatial environment}\label{fig:4C_gridworld}
\end{figure}

\section{Conclusions and further work}
%\section{CONCLUSIONS AND FURTHER WORK}

We have introduced TMDPs, a novel variant of MDPs. This is an original framework to support decision makers who confront adversaries that interfere with the reward generating process in reinforcement learning settings. TMDP's aim to provide one-sided prescriptive support to a DM, maximizing her subjective expected utility, taking into account potential negative actions taken by an adversary. Some theoretical results are provided, in particular, we proved that our proposed learning rule is a contraction mapping so that we may use standard RL results of convergence. In addition, we propose a scheme to model adversarial behavior based on level-$k$ reasoning about opponents. Further empirical evidence is provided via extensive experiments, with encouraging results.

Several lines of work are possible for further research. First of all, we
have limited to the case of facing just one adversary. The framework could be 
extended to the case of having multiple adversaries. In the experiments, we have just considered up to level-2 DM's, though the extension to higher order adversaries seems straightforward.

In addition, in recent years Q-learning has benefited from advances from the deep learning community, with breakthroughs such as the \emph{deep Q-network} (DQN) which achieved super-human performance in control tasks such as Atari games \cite{mnih2015human}, or as inner blocks inside systems that play Go \cite{silver2017mastering}. Integrating these advances into the TMDP setting is another possible research path. In particular, the proposed Algorithm \ref{alg:l2ur} can be 
generalized to account for the use of deep Q-networks instead of tabular Q-learning.

Finally, it might be interesting to explore similar expansions of semi-MDPs, in order to perform Hierarchical RL or allow for time-dependent rewards and transitions between states.

 \section{ Acknowledgments}
 
 R.N. acknowledges support from the Spanish Ministry for his grant FPU15-03636, V.G. acknowledges support from grant FPU16-05034. DRI is grateful to the MINECO MTM2014-56949-C3-1-R project and the AXA-ICMAT Chair in Adversarial Risk Analysis. All authors acknowledge support from the Severo Ochoa Excellence Programme SEV-2015-0554.

% Delete the next command after reviews!!!!
%\newpage

{
\normalsize
%\section{References}
\bibliographystyle{aaai} % or try abbrvnat or unsrtnat
\bibliography{marl} 
}

\newpage

\appendix

{\Large {\bf APPENDIX}}

\section{PROOFS}
\label{app:proofs}
\begin{lemma}\label{lema:1}
Given $q: \mathcal{S} \times \mathcal{B} \times \mathcal{A} \rightarrow \mathbb{R}$, the following operator $\mathcal{H}$ is a contraction mapping

\begin{flalign*}
&(\mathcal{H}q) (s,b,a) = \sum_{s'} p(s'|s,b,a) \big[ r(s,b,a) +\\
&+\gamma \max_{a'} \mathbb{E}_{p(b'|s')} q(s',b',a') \big].
\end{flalign*}

\end{lemma}
\begin{proof}
We show that $\mathcal{H}$ is a contraction under the supremum norm, i.e., $\| \mathcal{H}q_1 - \mathcal{H}q_2 \|_{\infty} \leq \gamma \| q_1 - q_2 \|_{\infty}$.

\begin{flalign*}
& \| \mathcal{H}q_1 - \mathcal{H}q_2 \|_{\infty} = \\
&= \max_{s,b,a}  \lvert \sum_{s'} p(s'|s,b,a) \big[ r(s,b,a) + \gamma \max_{a'} \mathbb{E}_{p(b'|s')} q_1(s',b',a')\\
&-r(s,b,a) - \gamma \max_{a'} \mathbb{E}_{p(b'|s')} q_2(s',b',a') \big] \rvert = \\
&= \gamma \max_{s,b,a}  \lvert \sum_{s'} p(s'|s,b,a) \big[   \max_{a'} \mathbb{E}_{p(b'|s')} q_1(s',b',a')\\
&-  \max_{a'} \mathbb{E}_{p(b'|s')} q_2(s',b',a')\big] \rvert \leq \\
&= \gamma \max_{s,b,a} \sum_{s'} p(s'|s,b,a) \lvert  \max_{a'} \mathbb{E}_{p(b'|s')} q_1(s',b',a') \\
&- \max_{a'} \mathbb{E}_{p(b'|s')} q_2(s',b',a')\rvert \leq \\
&= \gamma \max_{s,b,a} \sum_{s'} p(s'|s,b,a)  \max_{a',z} \lvert \mathbb{E}_{p(b'|z)} q_1(z,b',a') \\
&- \mathbb{E}_{p(b'|z)} q_2(z,b',a') \rvert \leq \\
&= \gamma \max_{s,b,a} \sum_{s'} p(s'|s,b,a)  \max_{a',z,b'} \lvert  q_1(z,b',a') -  q_2(z,b',a') \rvert =\\
&= \gamma \max_{s,b,a} \sum_{s'} p(s'|s,b,a)  \|  q_1 -  q_2 \|_{\infty} = \\
&= \gamma \|  q_1 -  q_2 \|_{\infty}.
\end{flalign*}

%where the last inequality is a result of the convex hull of $(q_1 - q_2)(z,b',a')$.

\end{proof}

%We also need the following result.

\begin{lemma}
Let $\bar{q} : \mathcal{S} \times \mathcal{A} \rightarrow \mathbb{R}$. The following operator $\bar{\mathcal{H}}$ is a contraction mapping

\begin{align*}
\begin{split}
&(\bar{\mathcal{H}}\bar{q}) (s,a) = \\ &= \mathbb{E}_{p(b|s)} \left[ \sum_{s'} p(s'|s,b,a) \left( r(s,b,a) + \gamma \max_{a'} \bar{q}(s',a') \right) \right].
\end{split}
\end{align*}

\end{lemma}
\begin{proof}
Similar to Lemma \ref{lema:1}, using the property that $\mathbb{E}_{a} f(a) \leq \max_a f(a)$ for any distribution $p(a)$.
\end{proof}

% \begin{proof}
% Given a TMDP $\left( \mathcal{S}, \mathcal{A}, \mathcal{B}, \mathcal{T}, r, p_A \right)$ with known opponent model $p_A(b | s)$, it is possible to define an equivalent MDP $\left( \mathcal{S} \times \mathcal{B}, \mathcal{A}, \bar{\mathcal{T}}, \bar{r} \right)$ where

% $$
% \bar{r} = r : \mathcal{S} \times \mathcal{B} \times \mathcal{A} \rightarrow \mathbb{R}
% $$

% and for the transition dynamics $\bar{\mathcal{T}}$ we have the following factorized distribution

% $$ p(s', b' | s,a,b) := p(s'|s,a,b) p_A (b | s) $$.

% {\color{blue} Creo que un TMDP no se puede convertir a MDP aunque sea estacionario. No pasa nada, solo que la demo será más larga..}
% \end{proof}

% The previous result suggests using ... as a learning rule. However, we use the following variant (it's better), so we use the following Lemma.

\section{EXPERIMENT DETAILS}\label{sec:exp_det}

We describe hyperparameters and other technical details used to obtain the results.

\subsection{REPEATED MATRIX GAMES}

\subsubsection{Memoryless Repeated Matrix Games}

In all three games (IPD, ISH, IC) we considered a discount factor $\gamma = 0.96$, a total of max steps $T = 20000$, initial $\epsilon = 0.1$ and learning rate $\alpha = 0.3$. 

The FP-Q learner started the learning process with a Beta prior $\mathcal{B}(1,1)$.

\subsubsection{Repeated Matrix Games With Memory}

In the IPD game we considered a discount factor $\gamma = 0.96$, a total of max steps $T = 20000$, initial $\epsilon = 0.1$ and learning rate $\alpha = 0.05$. 

The FP-Q learner started the learning process with a Beta prior $\mathcal{B}(1,1)$.

\subsection{AI SAFETY GRIDWORLDS}

\subsubsection{Stateless Variant}

Rewards for the DM are $50, -50$ depending on her action and the target chosen by the adversary.

We considered a discount factor $\gamma = 0.8$ and a total of $5000$ episodes. For all three agents, the initial exploration parameter was set to $\epsilon = 0.1$ and learning rate $\alpha = 0.1$. The FP-Q learner with forget factor used $\lambda = 0.8$.

\subsubsection{Spatial Variant}

Episodes end at a maximum of 50 steps or agent arriving first at target 1 or 2. Rewards for the DM are $-1$ for performing any action (i.e., a step in some of the four possible directions) or $50, -50$ depending on the target chosen by the adversary.

We considered a discount factor $\gamma = 0.8$ and a total of $15000$ episodes. For the level-2 agent, initial $\epsilon_A = \epsilon_B = 0.99$ with decaying rules $\epsilon_A := 0.995\epsilon_A$ and $\epsilon_B := 0.9\epsilon_B$ every $10$ episodes and learning rates $\alpha_A = \alpha_B = 0.05$. For the independent Q-learner we set initial exploration rate $\epsilon = 0.99$ with decaying rule $\epsilon := 0.995\epsilon$ every $10$ episodes and learning rate $\alpha = 0.05$.

\section{ADDITIONAL RESULTS}\label{sec:adv_rs}

For the experiments from Section \ref{sec:statv} we tried other models for the opponent's rewards $r_B$. Instead of assuming a minimax setting ($r_B = r_A$), where $r_B \in \lbrace -50, 50 \rbrace$, we tried also two different scalings $r_B \in \lbrace -1, 1 \rbrace$ and $r_B \in \lbrace 0, 1 \rbrace $. These alternatives are displayed in Figure \ref{fig:4C_rs}, though we found they did not qualitatively affect the results.
%Maybe it will be significant with a DQN, as reported in \url{https://arxiv.org/pdf/1709.06560.pdf}.

 \begin{figure}%
 \centering
 \subfigure[Rewards $+1$ and $0$ for the adversary]{%
 \label{fig:4C_binary_adv}%
 \includegraphics[scale=0.5]{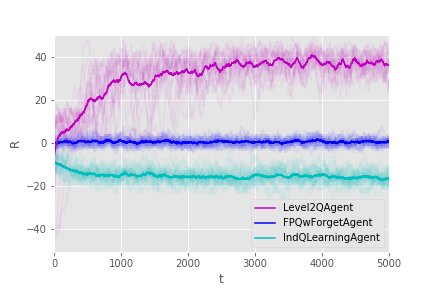}}%
 \vfill
 \subfigure[Rewards $+1$ and $-1$ for the adversary]{%
 \label{fig:4C_binary_1m1}%
 \includegraphics[scale=0.5]{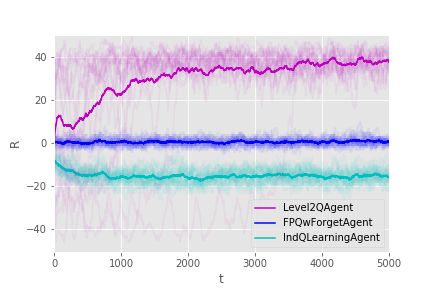}}%
 \caption{Rewards against the same adversary (exponential smoother) using different reward scalings.}\label{fig:4C_rs}
 \end{figure}

\end{document}